\documentclass{article}
\usepackage{iclr2027_conference,times}

\usepackage{amsmath,amsfonts,bm}

\def\eqref#1{equation~\ref{#1}}

\def\1{\bm{1}}

\DeclareMathAlphabet{\mathsfit}{\encodingdefault}{\sfdefault}{m}{sl}
\SetMathAlphabet{\mathsfit}{bold}{\encodingdefault}{\sfdefault}{bx}{n}

\newcommand{\KL}{D_{\mathrm{KL}}}

\usepackage[T1]{fontenc}
\usepackage[utf8]{inputenc}
\usepackage{graphicx,booktabs,array,tabularx,makecell,multirow}
\usepackage{flafter,placeins,needspace}
\usepackage[table]{xcolor}
\usepackage{amsmath,amssymb,amsthm}
\usepackage{enumitem,microtype,url}
\usepackage{hyperref}

\definecolor{red}{HTML}{CF553D}
\definecolor{oursgroup}{HTML}{E8F0FE}
\hypersetup{
  colorlinks=true,
  linkcolor=red,
  filecolor=magenta,
  urlcolor=cyan,
  citecolor=cyan,
  pdftitle={ReCal: Calibrating Structured Pruning for On-Policy Distillation Recovery},
  pdfauthor={Houcheng Jiang, Mao Zheng, Mingyang Song, Qiyong Zhong, Jie Sun, Tianyu Zhang, Junfeng Fang}
}

\newcolumntype{Y}{>{\centering\arraybackslash}X}
\setlist{nosep,leftmargin=*}
\renewcommand{\paragraph}[1]{\par\textbf{#1}\hspace{1em}}
\newtheorem{proposition}{Proposition}
\renewcommand{\KL}{\mathop{\mathrm{KL}}}
\newcommand{\method}{\textsc{ReCal}}
\DeclareMathOperator{\TV}{TV}

\iclrfinalcopy

\title{\textsc{ReCal}: Calibrating Structured Pruning for\\
On-Policy Distillation Recovery}

\author{%
\textbf{Houcheng Jiang}$^{1,2,3,*}$,
\textbf{Mao Zheng}$^{2,*}$,
\textbf{Mingyang Song}$^{2,*}$,
\textbf{Qiyong Zhong}$^{2}$\\
\textbf{Jie Sun}$^{3}$,
\textbf{Tianyu Zhang}$^{3}$,
\textbf{Junfeng Fang}$^{4,\dagger}$\\[2pt]
\normalfont $^{1}$Zhongguancun Academy\quad
$^{2}$Foundation Model Department, Tencent\\
\normalfont $^{3}$University of Science and Technology of China\\
\normalfont $^{4}$National University of Singapore\\[2pt]
\normalfont $^{*}$\,Equal Contribution\quad
$^{\dagger}$\,Corresponding author
}

\begin{document}
\raggedbottom
\maketitle

\fancyhead{}
\fancyhead[L]{Preprint.}
\fancyfoot{}
\fancyfoot[C]{\thepage}
\renewcommand{\headrulewidth}{0pt}
\renewcommand{\footrulewidth}{0pt}

\begin{abstract}
Structured pruning reduces the deployment cost of reasoning language models, but capability degradation can hinder subsequent on-policy distillation (OPD) recovery. Because OPD relies on student-generated trajectories, damage that persists after offline distillation can limit its effectiveness. We propose \method{}, Recovery-Aware Calibration, a simple plug-and-play approach that improves OPD recovery by adjusting calibration before pruning. \method{} uses forward KL between an unpruned teacher and a pruned probe to identify disrupted teacher-supported predictions, then reweights calibration statistics to guide existing pruning criteria toward preserving them. Across models and pruning methods, \method{} improves mathematical reasoning after OPD in most comparisons, with gains of up to 16.7 percentage points on AIME and improvements in most code-generation comparisons. Further analysis shows reduced residual damage at heavily affected tokens and performance advantages that persist through recovery. These results support calibrating structured pruning for OPD recovery. Our code: \url{https://github.com/jianghoucheng/ReCal}.
\end{abstract}

\section{Introduction}\label{sec:intro}
Large language models (LLMs) increasingly support complex reasoning and autonomous agent tasks, but their growing size and long generations make deployment expensive \citep{anthropic2026fable,openai2026astra,yang2025qwen3}. Structured pruning reduces this cost by removing entire channels, heads, or layers, producing smaller models that run on standard hardware \citep{ma2023llmpruner,an2023flap,ashkboos2024slicegpt}. However, pruning often requires subsequent training to recover lost capabilities \citep{xia2023sheared,muralidharan2024compact,sreenivas2024minitron}. Recent distillation advances support a two-stage recovery pipeline: offline distillation through supervised fine-tuning (SFT) on teacher responses, followed by on-policy distillation (OPD) on student-generated trajectories \citep{kim2016sequence,agarwal2023opd,zhang2026shortopd}. The unpruned model serves as teacher and the pruned model as student.

\begin{figure}[t]
\centering
\includegraphics[width=\linewidth]{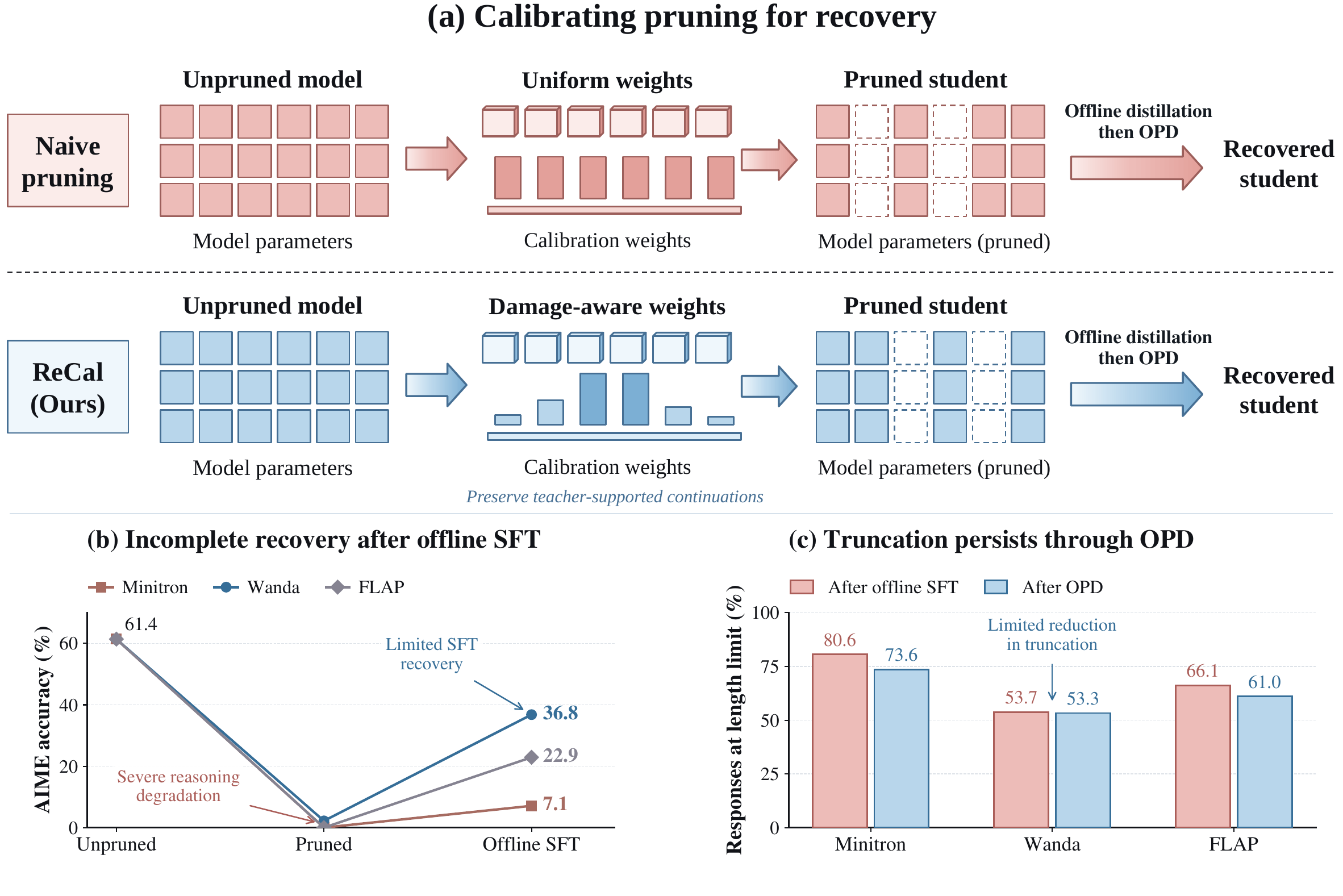}
\caption{\textbf{Pruning and the limitations of subsequent recovery.} (a) \method{} adjusts calibration before SFT and OPD. (b) AIME accuracy before pruning, after pruning, and after offline SFT. (c) The fraction of responses reaching the length limit after SFT and OPD. Both data panels show measured 8B baselines; Appendix~\ref{app:intro_design} gives the protocol.}
\label{fig:intro}
\end{figure}

Despite this progress, pruning can severely degrade reasoning, and offline SFT only partly restores accuracy (Figure~\ref{fig:intro}b). Pruned models can produce repetitive, unfinished, or incorrect responses \citep{wang2025think,zhang2026shortopd}. Worse still, OPD obtains teacher supervision at these student-generated prefixes, where poor generations can compromise the feedback \citep{xu2025speculative}. A teacher can agree with the student's next-token predictions at a flawed prefix even when the reasoning fails; teacher-confirmed repetition is one documented example \citep{zhang2026shortopd}. Low reverse KL on student rollouts therefore need not indicate successful reasoning. In our baseline evaluations, many responses still reach the length limit after OPD (Figure~\ref{fig:intro}c), leaving substantial generation failures unresolved.

Since recovery depends on student trajectories, pruning should account for damage that may hinder recovery, complementing training-aware structure selection \citep{tang2026darwinlm,ai2025nirvana}. We examine this damage on teacher-generated prefixes, where the teacher and a pruned model receive the same context. Forward KL emphasizes predictions that the teacher supports but the pruned model assigns little probability. Our analysis shows that tokens most affected by pruning remain a major source of discrepancy after recovery. Giving these tokens greater calibration weight could therefore protect vulnerable predictions before recovery begins.

We therefore propose \method{}, a simple \emph{Recovery-Aware Calibration} method for structured pruning. As shown in Figure~\ref{fig:intro}(a), \method{} first uses a pruned probe to measure token-level forward KL on teacher trajectories. These scores reweight calibration statistics used to recompute channel importance on the unpruned model. Applying the new mask produces a student for the same SFT and OPD recovery pipeline. \method{} is plug-and-play: it preserves the base pruning criterion and target widths while changing the evidence used to select channels.

To evaluate \method{}, we compare its recovery performance with representative structured pruning baselines on Qwen3-8B and Qwen3-4B-Instruct. After on-policy distillation, \method{} achieves stronger mathematical reasoning than the best-performing baseline on both models. It also improves mathematical reasoning in seven of eight model--criterion comparisons, with gains of up to 16.7 percentage points, and benefits code generation in most comparisons. These results show that calibrating which structures survive pruning can improve the capabilities recovered by the subsequent distillation pipeline.

\section{Preliminaries}\label{sec:prelim}
We first introduce calibration-based structured pruning in Section~\ref{sec:pruning}, then describe the on-policy distillation recovery process that follows it in Section~\ref{sec:opd}.

\subsection{Structured Pruning}\label{sec:pruning}
Structured pruning removes parameter groups to obtain a smaller model \citep{ma2023llmpruner,an2023flap,muralidharan2024compact}. We consider feed-forward network (FFN) width pruning, where each selected channel is removed together with its associated input and output weights. Let $p_T$ denote the original model and $m^\ell\in\{0,1\}^{d_\ell}$ the channel mask in layer $\ell$. A value of one retains a channel. Applying these masks to the original weights gives the pruned model $p_m$.

Calibration-based criteria use a small set of inputs to estimate channel importance. For an additive criterion, let $c_{ij}^\ell$ be the contribution of calibration token $i$ to channel $j$ in layer $\ell$. Aggregating over $N$ tokens and retaining the $k_\ell$ highest-scoring channels gives
\begin{equation}\label{eq:base}
s_j^\ell=\frac1N\sum_{i=1}^N c_{ij}^\ell,
\qquad m_j^\ell=\mathbb{1}\bigl[j\in\operatorname{Top}_{k_\ell}(s^\ell)\bigr].
\end{equation}
The contribution can use activation magnitude, norms, fluctuations, or gradients \citep{sun2023wanda,an2023flap,ma2023llmpruner}. Equation~\ref{eq:base} covers additive scores; see Appendix~\ref{app:adapters} for other criteria.

We reweight token evidence while preserving the criterion and target widths.

\subsection{On-Policy Distillation Recovery}\label{sec:opd}
Offline sequence distillation first initializes recovery by fitting fixed responses from the unpruned teacher \citep{kim2016sequence}. We implement this stage as supervised fine-tuning (SFT) on the teacher responses; we use SFT to denote this offline distillation stage throughout the paper. On-policy distillation (OPD) then trains the student on responses sampled from its own policy, with a frozen teacher providing the supervision \citep{agarwal2023opd,li2026rethinking}. Let $p_\theta$ be the trainable pruned student. For prompt $x$, a rollout checkpoint $p_{\bar\theta}$ generates response $y$, and both teacher and student evaluate the prefixes $h_t=(x,y_{<t})$. A full-vocabulary reverse-KL matching objective on these rollouts is
\begin{equation}\label{eq:opd}
\mathcal L_{\mathrm{OPD}}(\theta;\bar\theta)=
\mathbb E_{x,\,y\sim p_{\bar\theta}(\cdot\mid x)}
\left[\frac1{|y|}\sum_t\KL\bigl(p_\theta(\cdot\mid h_t)\Vert p_T(\cdot\mid h_t)\bigr)\right],
\end{equation}
where $\KL(p\Vert q)=\sum_v p(v)\log[p(v)/q(v)]$. Fresh responses are collected as the student updates, so the teacher supervises prefixes reached by the student's current policy.

The retained structure determines the starting student for this process. We hold recovery fixed and examine whether pruning damage can inform structure selection. Appendix~\ref{app:protocol} specifies the practical OPD signal and training settings.

\section{Understanding Pruning Damage and Recovery}\label{sec:analysis}
\begin{figure}[t]
\centering
\includegraphics[width=\linewidth]{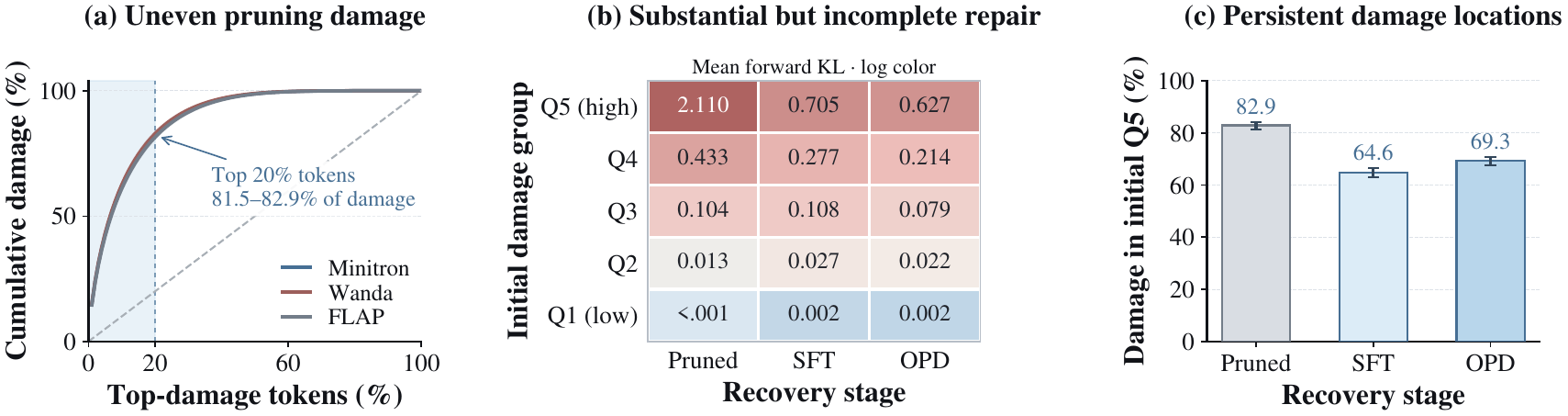}
\caption{\textbf{Locating damage before recovery.} (a) Within-trajectory damage concentration across ordinary pruning criteria. (b) Mean KL through baseline recovery in groups fixed by initial damage; color is logarithmic. (c) The fraction of each response's damage carried by its initially highest-damage group (Q5), averaged across responses. Q1--Q5 are equal-count groups ordered by initial damage. Token groups are never reselected after recovery. SFT denotes offline distillation through supervised fine-tuning. The pale vertical band marks Q5; curve bands and error bars denote 95\% trajectory-bootstrap intervals.}
\label{fig:analysis}
\end{figure}

We analyze several structure pruning methods on held-out teacher responses to locate pruning damage and track its recovery. At token $t$ in response $r$, both models receive the same prefix $h_{r,t}$. Token-level damage is
\begin{equation}\label{eq:damage}
d_{r,t}(m)=\KL\bigl(p_T(\cdot\mid h_{r,t})\Vert p_m(\cdot\mid h_{r,t})\bigr).
\end{equation}
This measures the change in next-token predictions, weighted by the teacher distribution. Within each response, we rank tokens by raw-pruned baseline KL and divide them into five equal-count groups, Q1--Q5, from lowest to highest damage. Each group contains one fifth of the response's tokens. We keep memberships and teacher prefixes fixed through offline SFT and OPD, averaging KL within each response/group and then across responses. Concentration is examined in both models; recovery tracking uses Wanda. Full settings are in Appendix~\ref{app:diagnostics}.

\noindent\textbf{Pruning distorts some token predictions much more than others.}\label{sec:damage_concentration}
Figure~\ref{fig:analysis}(a) shows that the highest-damage 20\% of tokens contribute 81.5--82.9\% of within-response damage across the three criteria. The second model exhibits the same pattern, with shares of 80.9--83.1\% (Appendix~\ref{app:diagnostics}). Channel removal therefore affects token predictions unevenly, which uniform aggregation does not explicitly distinguish. We next ask whether these predictions remain damaged after recovery.

\noindent\textbf{Recovery reduces token-level damage but leaves substantial residual errors.}\label{sec:damage_recovery}
Using the fixed groups defined above, Figure~\ref{fig:analysis}(b) shows that mean KL in the highest-damage group decreases from 2.110 after pruning to 0.705 after SFT and 0.627 after OPD. The full pipeline removes 70.3\% of this group's initial discrepancy, with most of that reduction occurring during SFT. Nevertheless, its post-OPD KL remains nearly three times that of the next-highest group (0.214). Substantial repair therefore leaves these tokens disproportionately damaged.

\noindent\textbf{Tokens most affected by pruning still account for most damage after recovery.}
To quantify how much of the remaining damage can be located before recovery, we measure the share carried by the original highest-damage group. Without reselecting any tokens, this group accounts for 64.6\% of within-response damage after SFT and 69.3\% after OPD (Figure~\ref{fig:analysis}c). The increase in share from SFT to OPD does not mean its damage grows: its absolute KL decreases, while its share of the remaining discrepancy rises. Initial damage thus identifies tokens that remain poorly matched after recovery, using a signal already available before training.

These observations motivate accounting for teacher-prefix damage during structure selection. A base-pruned model can identify affected tokens; emphasizing their calibration evidence may reduce the discrepancy passed to recovery. Section~\ref{sec:method} defines this intervention, and Section~\ref{sec:experiments} tests whether it improves recovered accuracy.

\section{Recovery-Aware Calibration for Pruning}\label{sec:method}
\method{} adjusts calibration to address the persistent damage identified in Section~\ref{sec:analysis}. We first locate token-level damage, use it to calibrate structure selection, and then recover the selected student through SFT and OPD. Figure~\ref{fig:method} summarizes this procedure.

\begin{figure}[t]
\centering
\includegraphics[width=\linewidth]{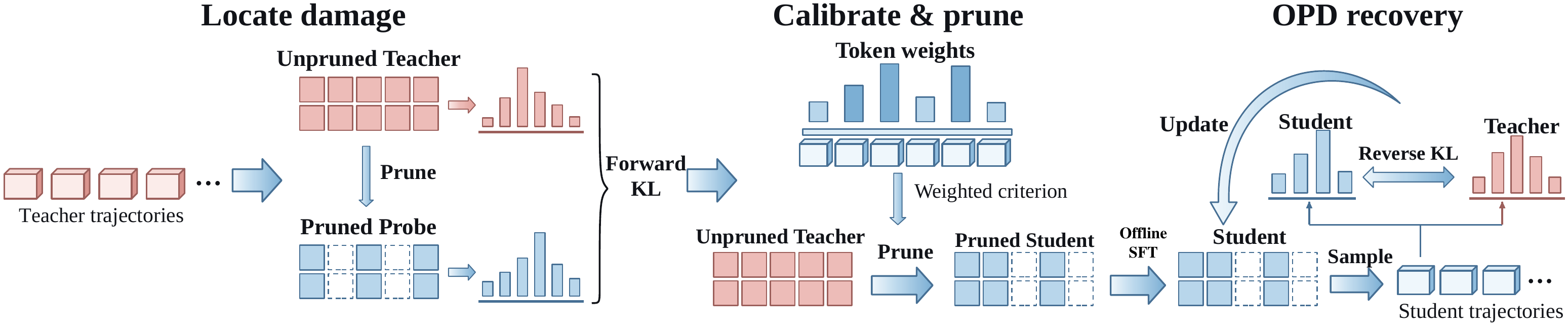}
\caption{\textbf{Overview of \method{}.} \emph{Locate damage:} compare the teacher and a pruned probe on teacher trajectories. \emph{Calibrate \& prune:} use token-level damage to reweight the base criterion and prune the original model. \emph{OPD recovery:} recover the selected student with the unchanged training procedure.}
\label{fig:method}
\end{figure}

\paragraph{Locate damage.}
We apply the base pruning criterion at the target widths to obtain a probe $p_{m_0}$. The teacher and probe evaluate identical prefixes from teacher-generated responses. We estimate the forward-KL damage in Equation~\ref{eq:damage} using teacher-selected tokens and a residual-probability bucket (Appendix~\ref{app:adapters}), and denote the resulting scores by $d_{r,t}$. Teacher trajectories expose damage independently of the probe's own rollouts. The probe does not restrict the final mask.

\paragraph{Calibrate \& prune.}
We convert the probe's damage into calibration weights. For response $r$ with $T_r$ valid tokens, let $D_r=\sum_t d_{r,t}$. The token weights and resulting additive channel scores are
\begin{equation}\label{eq:weights}
\widetilde w_{r,t}=\frac{d_{r,t}}{D_r},\qquad
\sum_{t=1}^{T_r}\widetilde w_{r,t}=1.
\end{equation}
\begin{equation}\label{eq:recal}
s_j^\ell=\frac1R\sum_{r=1}^R\sum_{t=1}^{T_r}\widetilde w_{r,t}c_{r,t,j}^\ell,
\qquad(m^\star)_j^\ell=\mathbb{1}\bigl[j\in\operatorname{Top}_{k_\ell}(s^\ell)\bigr].
\end{equation}
Here $R$ is the number of calibration responses and $c_{r,t,j}^\ell$ is the base contribution defined in Section~\ref{sec:pruning}. Each response has equal total weight, with greater weight assigned to its more damaged tokens. For negligible total damage, we use uniform weights (Appendix~\ref{app:adapters}).

The weighted statistics are computed on the original unpruned model, and the selected mask is applied to its original weights. Thus, a channel removed by the probe remains eligible for the final student. Non-additive criteria incorporate the weights into their activation statistics or token losses before aggregation (Appendix~\ref{app:adapters}). The base criterion and retained widths remain unchanged. Appendix~\ref{app:theory} connects teacher-prefix forward KL to sequence preservation and states its limits.

\paragraph{OPD recovery.}
The selected student first undergoes offline distillation through SFT on fixed teacher responses and then OPD with the same unpruned teacher. Both stages use the same training procedure as the baseline. \method{} changes the structure supplied to recovery, while OPD continues to learn from student-generated trajectories as in Section~\ref{sec:opd}. Appendix~\ref{app:protocol} gives training details.

\section{Experiments}\label{sec:experiments}
\begin{table}[!t]
\centering\fontsize{10}{11.3}\selectfont
\setlength{\tabcolsep}{2.2pt}\renewcommand{\arraystretch}{1.0}
\caption{\textbf{Recovery at approximately 25\% parameter reduction.} \method{} denotes Wanda + \method{}. Pruned: no recovery; SFT: offline distillation; OPD: SFT followed by OPD. Scores are percentages; Avg. averages all six benchmarks. Each column's best/second-best pruned scores within a model are \textbf{bold}/\underline{underlined}, excluding unpruned references.}
\label{tab:main}
\begin{tabularx}{\linewidth}{@{}ll*{7}{Y}@{}}
\toprule
\multirow{2}{*}{Method} & \multirow{2}{*}{Stage} & \multicolumn{3}{c}{Math} & Science & \multicolumn{1}{c}{Instruction} & Code & \multirow{2}{*}{Avg.}\\
\cmidrule(lr){3-5}\cmidrule(lr){6-6}\cmidrule(lr){7-7}\cmidrule(lr){8-8}
& & AIME24 & AIME25 & AIME26 & GPQA & IFB-S & LCB-v5 & \\

\midrule
\multicolumn{9}{c}{\textbf{Qwen3-8B}}\\
\midrule
Unpruned & --- & 66.2 & 55.8 & 62.1 & 54.5 & 23.0 & 46.4 & 51.3\\
\midrule
Minitron & Pruned & 0.0 & 0.0 & 0.0 & 1.0 & 16.9 & 0.0 & 3.0\\
 & SFT & 9.2 & 3.3 & 8.7 & 30.3 & 18.0 & 15.7 & 14.2\\
 & OPD & 16.7 & 10.0 & 15.4 & 35.9 & 21.2 & 19.9 & 19.9\\
\addlinespace[1.5pt]
FLAP & Pruned & 0.0 & 0.0 & 0.0 & 23.2 & 16.3 & 0.0 & 6.6\\
 & SFT & 27.5 & 20.8 & 20.4 & 33.3 & 19.2 & 15.1 & 22.7\\
 & OPD & 31.7 & 25.4 & 30.0 & \underline{43.4} & 20.1 & 24.7 & 29.2\\
\addlinespace[1.5pt]
LLM-Pruner & Pruned & 0.4 & 0.0 & 0.0 & 22.2 & 16.0 & 0.0 & 6.4\\
 & SFT & 11.7 & 6.2 & 5.8 & 35.9 & 19.5 & 12.7 & 15.3\\
 & OPD & 18.8 & 7.9 & 10.4 & 34.8 & \textbf{22.4} & 17.5 & 18.6\\
\addlinespace[1.5pt]
Wanda & Pruned & 2.5 & 2.1 & 2.1 & 26.8 & 20.6 & 4.2 & 9.7\\
 & SFT & 41.7 & 34.6 & 34.2 & \textbf{45.5} & \underline{21.5} & 25.3 & 33.8\\
 & OPD & 46.7 & 36.7 & 41.7 & 42.4 & 19.5 & 25.9 & 35.5\\
\addlinespace[1.5pt]
\rowcolor{oursgroup}\method{} & Pruned & 33.3 & 23.7 & 20.0 & 35.4 & 20.1 & 16.9 & 24.9\\
\rowcolor{oursgroup} & SFT & \underline{51.7} & \underline{39.6} & \underline{46.3} & \underline{43.4} & \underline{21.5} & \underline{26.5} & \underline{38.2}\\
\rowcolor{oursgroup} & OPD & \textbf{52.9} & \textbf{42.9} & \textbf{49.6} & 40.4 & 18.3 & \textbf{28.3} & \textbf{38.7}\\
\midrule
\multicolumn{9}{c}{\textbf{Qwen3-4B-Instruct}}\\
\midrule
Unpruned & --- & 63.3 & 45.8 & 54.2 & 60.6 & 34.3 & 31.9 & 48.4\\
\midrule
Minitron & Pruned & 0.0 & 0.4 & 0.0 & 30.3 & 26.7 & 1.2 & 9.8\\
 & SFT & 6.2 & 0.8 & 5.4 & 42.9 & 28.2 & 15.7 & 16.5\\
 & OPD & 6.2 & 2.1 & 6.2 & 35.4 & \underline{30.2} & 21.7 & 17.0\\
\addlinespace[1.5pt]
FLAP & Pruned & 4.2 & 0.8 & 0.8 & 31.8 & 27.9 & 4.2 & 11.6\\
 & SFT & 21.7 & 11.7 & 13.8 & 36.9 & 26.2 & 21.7 & 22.0\\
 & OPD & 27.9 & 17.5 & 19.2 & 38.9 & \textbf{30.8} & 22.3 & 26.1\\
\addlinespace[1.5pt]
LLM-Pruner & Pruned & 0.0 & 0.0 & 0.0 & 26.8 & 25.0 & 0.6 & 8.7\\
 & SFT & 6.7 & 2.5 & 6.7 & 31.8 & 25.9 & 15.1 & 14.8\\
 & OPD & 8.7 & 0.8 & 8.7 & 35.9 & 27.9 & 15.1 & 16.2\\
\addlinespace[1.5pt]
Wanda & Pruned & 13.3 & 5.4 & 6.7 & 40.4 & 29.7 & 5.4 & 16.8\\
 & SFT & 29.6 & 23.7 & 20.4 & \textbf{47.0} & 27.3 & 20.5 & 28.1\\
 & OPD & 35.8 & \underline{32.5} & 28.7 & \underline{44.4} & 29.7 & \underline{24.7} & 32.6\\
\addlinespace[1.5pt]
\rowcolor{oursgroup}\method{} & Pruned & 20.4 & 10.4 & 10.8 & 37.9 & \underline{30.2} & 15.1 & 20.8\\
\rowcolor{oursgroup} & SFT & \underline{41.7} & 31.7 & \underline{29.2} & 43.9 & 27.3 & 22.9 & \underline{32.8}\\
\rowcolor{oursgroup} & OPD & \textbf{42.5} & \textbf{34.2} & \textbf{35.0} & \textbf{47.0} & 28.8 & \textbf{25.3} & \textbf{35.5}\\
\bottomrule
\end{tabularx}
\end{table}

After introducing the setup in Section~\ref{sec:setup}, we address four questions in Sections~\ref{sec:main_results}--\ref{sec:ablation}:
\begin{itemize}[itemsep=2pt]
\item \textbf{RQ1:} How does \method{} compare with existing pruning methods before and after recovery? Does it improve the performance reached after OPD?
\item \textbf{RQ2:} Can \method{} improve different pruning criteria without changing their scoring rules? Do its benefits extend beyond the default Wanda configuration?
\item \textbf{RQ3:} How does the benefit of \method{} change as more parameters are removed? Does it persist across compression ratios, models, and base criteria?
\item \textbf{RQ4:} How does \method{} change the selected channels and the damage identified in Section~\ref{sec:analysis}? How do these changes relate to recovery?
\end{itemize}

\subsection{Experimental Setup}\label{sec:setup}
Appendix~\ref{app:protocol} details the pruning adapters, data sources, training, and evaluation.

\paragraph{Models and baselines.}
We use Qwen3-8B in thinking mode and Qwen3-4B-Instruct-2507 without thinking \citep{yang2025qwen3}. We compare structured FFN-channel adapters of Minitron, Wanda, FLAP, and LLM-Pruner \citep{muralidharan2024compact,sun2023wanda,an2023flap,ma2023llmpruner}. All methods retain the same widths at a given compression ratio. The main comparison removes approximately 25\% of total parameters; the ratio sweep also includes 15\% and 35\%.

\paragraph{Recovery and evaluation.}
Every student undergoes offline distillation through SFT on responses from its unpruned teacher, followed by OPD on student-generated trajectories. Paired methods share recovery data and training settings. We evaluate mathematics on AIME 2024--2026, science on GPQA-Diamond \citep{rein2023gpqa}, instruction following on IFBench \citep{pyatkin2025ifbench}, and code generation on LiveCodeBench-v5 \citep{jain2024livecodebench}. We report AIME accuracy averaged over eight samples per question and the three years, GPQA accuracy, IFBench instruction-level strict accuracy, and LiveCodeBench pass@1. Supplementary experiments and additional results appear in Appendix~\ref{app:supplementary}.

\begin{figure}[t]
\centering\includegraphics[width=\linewidth]{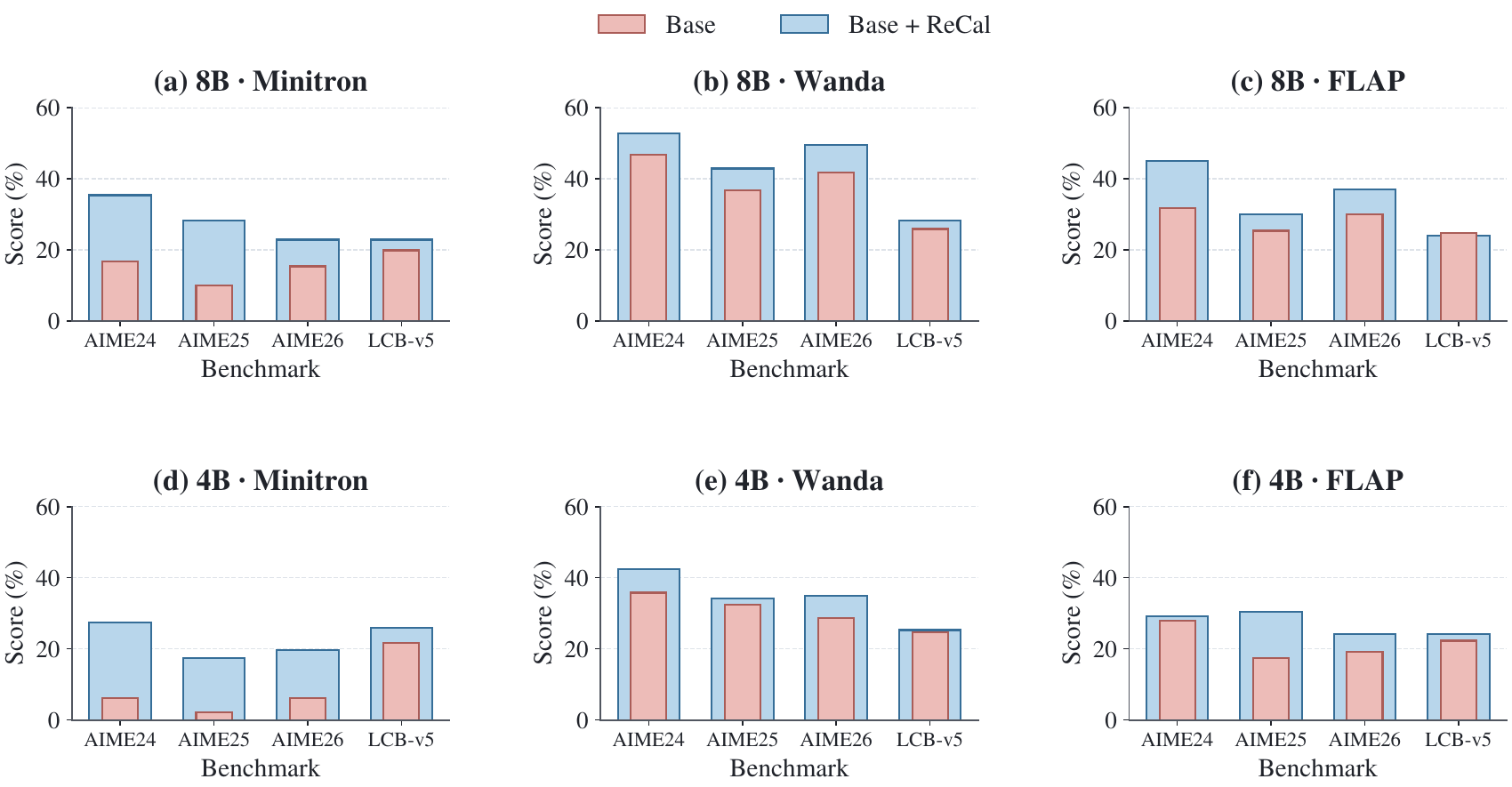}
\caption{\textbf{Post-OPD performance across pruning criteria.} Measured results on AIME24--26 and LiveCodeBench-v5 (LCB-v5). Wide blue bars show Base + \method{}; narrow red bars show Base. Rows identify models and columns identify criteria. See Appendix~\ref{app:llmpruner} for LLM-Pruner.}
\label{fig:compatibility}
\end{figure}
\subsection{Recovery Performance (RQ1)}\label{sec:main_results}
We use Wanda as the base criterion for the default \method{} configuration. Table~\ref{tab:main} compares the three stages: Pruned, SFT, and SFT followed by OPD. We have the following observations:

\begin{itemize}[itemsep=3pt,topsep=2pt]
\item \textbf{Obs. 1: ReCal leads in mathematics after recovery.}
After OPD, \method{} reaches 48.5\% AIME accuracy in 8B and 37.2\% in 4B, exceeding the strongest baseline, Wanda, at 41.7\% and 32.4\%. It leads on each AIME year in both models. After SFT, it already reaches 45.8\% and 34.2\%, respectively, establishing an advantage that persists through OPD.

\item \textbf{Obs. 2: Calibration preserves reasoning capability before recovery.}
Immediately after pruning, the 8B default retains 25.7\% AIME accuracy, compared with 2.2\% for Wanda. The 4B scores are 13.9\% and 8.5\%. The advantage is already present before any recovery training, consistent with calibration preserving more of the teacher's reasoning capability in the selected structure.
\end{itemize}

\subsection{Plug-and-Play Improvement (RQ2)}\label{sec:compatibility}
We next apply \method{} to each pruning criterion while keeping its channel-scoring rule, target widths, and recovery procedure fixed. Figure~\ref{fig:compatibility} compares each fully recovered baseline with its calibrated counterpart. Baselines use prompt calibration, whereas \method{} uses weighted teacher trajectories; the comparison evaluates this combined intervention. We can draw the following observations:

\begin{itemize}[itemsep=3pt,topsep=2pt]
\item \textbf{Obs. 3: ReCal improves mathematics in seven of eight pairs.}
All six activation-based pairs improve on AIME (Figure~\ref{fig:compatibility}), with gains up to 16.7 points. LLM-Pruner improves from 6.1\% to 12.8\% in 4B but declines from 12.4\% to 9.2\% in 8B (Appendix~\ref{app:llmpruner}). The benefit therefore extends across criteria, with a model-dependent exception.

\item \textbf{Obs. 4: Code improves broadly, with trade-offs on other tasks.}
LiveCodeBench improves in seven of eight pairs and GPQA in five, while IFBench declines in all eight (Figure~\ref{fig:recovery_evidence}c). For 8B LLM-Pruner, science and code gain 10.1 and 10.8 points despite the mathematics loss. These trade-offs limit claims of uniform capability preservation (Appendices~\ref{app:llmpruner} and~\ref{app:additional_results}).
\end{itemize}

\subsection{Sensitivity to Compression Ratio (RQ3)}\label{sec:ratios}
\begin{figure}[t]
\centering\includegraphics[width=\linewidth]{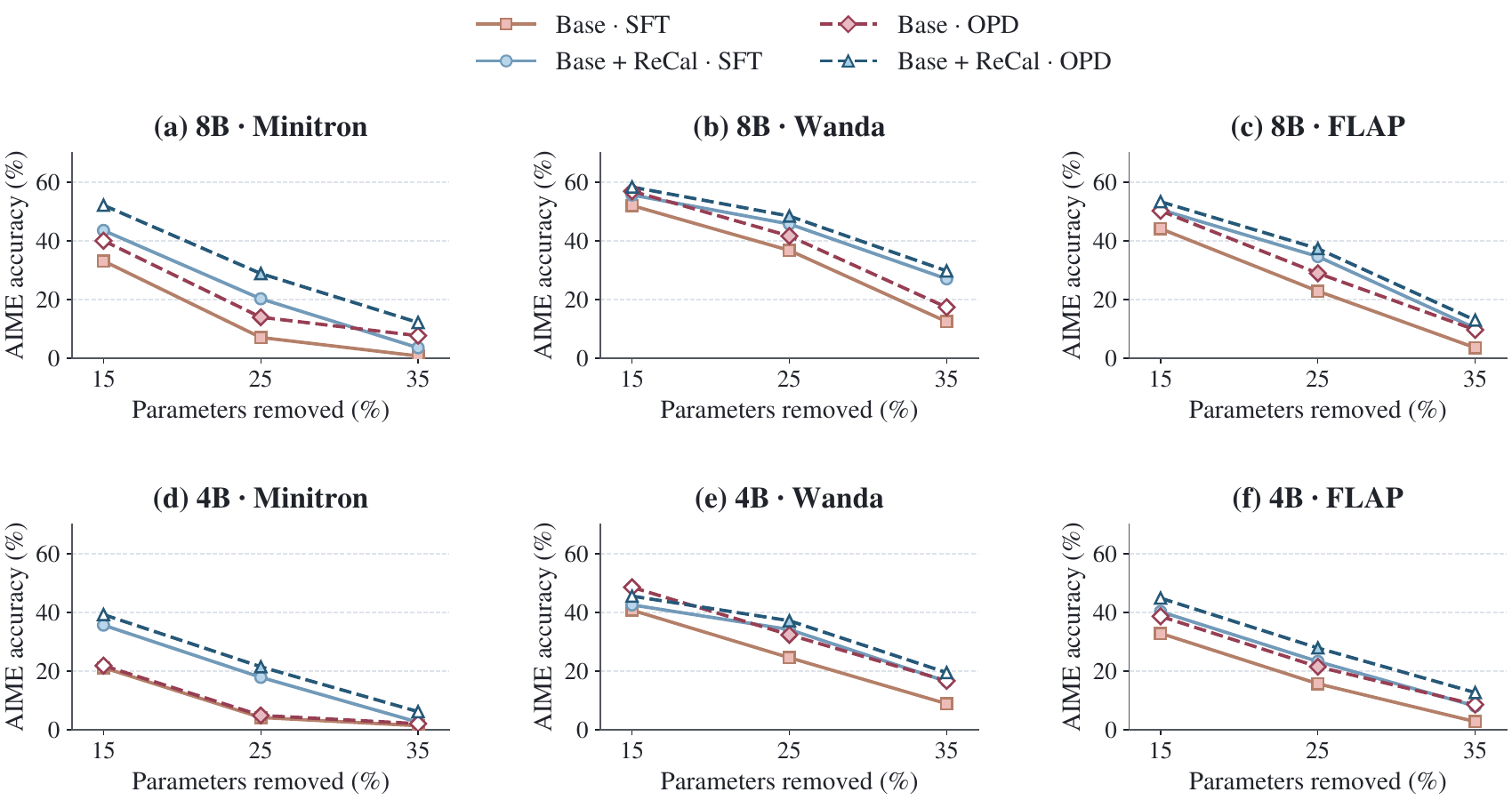}
\caption{\textbf{Base versus Base + \method{} across compression ratios.} Solid: SFT; dashed: OPD. Ratios denote whole-model parameter reduction.}
\label{fig:ratios}
\end{figure}
We vary whole-model parameter reduction from 15\% to 35\% for Minitron, Wanda, and FLAP in both models. Figure~\ref{fig:ratios} shows both recovery stages. We can draw the following observations:

\begin{itemize}[itemsep=3pt,topsep=2pt]
\item \textbf{Obs. 5: ReCal improves SFT recovery at every measured ratio.}
\method{} exceeds Base in all 18 model--criterion--ratio comparisons. Wanda gains 3.6, 9.0, and 14.6 points in 8B at 15\%, 25\%, and 35\% reduction; 4B gains are 1.8, 9.6, and 7.6 points. Benefits persist across ratios but need not increase monotonically.

\item \textbf{Obs. 6: More aggressive pruning still limits recovery.}
All measured SFT curves decline with compression. At 35\% reduction, 8B Wanda + \method{} retains 27.1\% accuracy versus 12.5\% for Base, while calibrated Minitron reaches only 3.6\%. In 4B, calibrated Wanda also outperforms calibrated Minitron and FLAP at this ratio. Both the criterion and retained capacity still matter.
\end{itemize}

\subsection{Structure and Recovery Analysis (RQ4)}\label{sec:ablation}
\begin{figure}[t]
\centering\includegraphics[width=\linewidth]{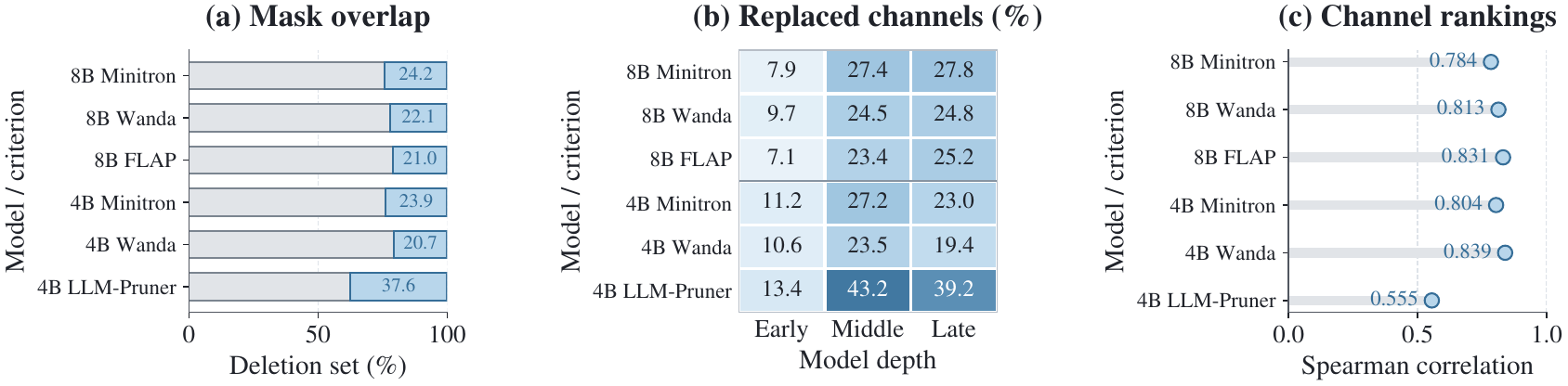}
\caption{\textbf{Changes in channel selection.} (a) Shared (gray) and replaced (blue) baseline deletions. (b) Replacement by layer depth. (c) Spearman correlation of complete channel rankings.}
\label{fig:structure}
\vspace{7pt}
\centering\includegraphics[width=\linewidth]{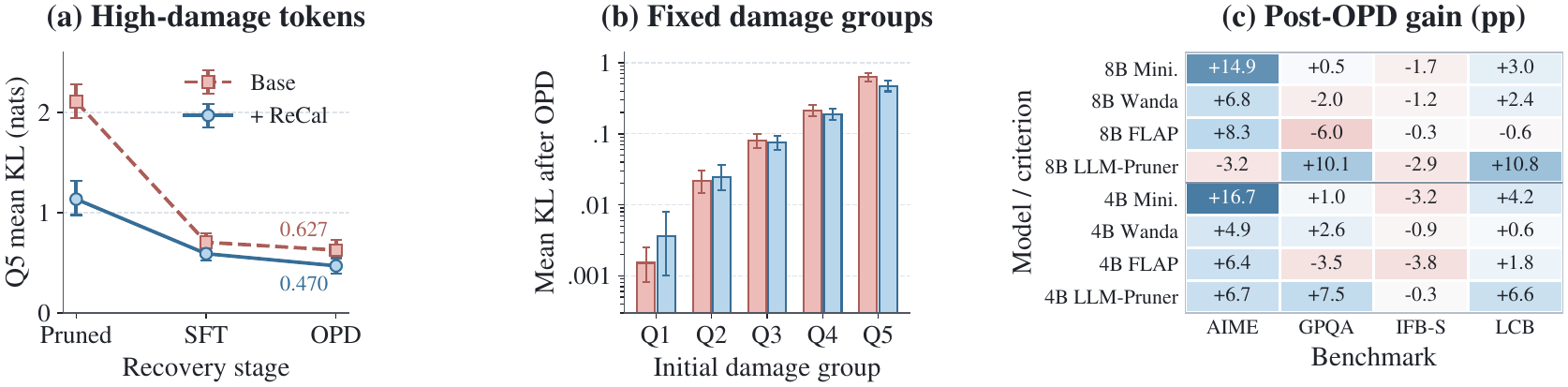}
\caption{\textbf{Damage reduction and recovery.} (a) Q5 KL through recovery. (b) Post-OPD KL by fixed group (log scale). Both use 8B Wanda and Section~\ref{sec:analysis} teacher prefixes; error bars are 95\% trajectory-bootstrap intervals. (c) Post-OPD gains across eight pairs and four tasks: blue indicates gains and red indicates losses. AIME is the three-year mean.}
\label{fig:recovery_evidence}
\end{figure}
We compare masks at identical widths (Figure~\ref{fig:structure}) and revisit the teacher prefixes and baseline-defined Q1--Q5 groups from Section~\ref{sec:analysis}, now adding the calibrated checkpoints (Figure~\ref{fig:recovery_evidence}). We have the following observations:

\begin{itemize}[itemsep=3pt,topsep=2pt]
\item \textbf{Obs. 7: ReCal preserves most selections, with larger changes in deeper layers.}
Activation-based pairs retain 75.8--79.3\% of baseline deletions, with ranking correlations of 0.784--0.839. Replacement is greater in middle and late layers; 4B LLM-Pruner replaces 37.6\% of its deletion set. With widths fixed, \method{} changes which channels survive while retaining most base selections.

\item \textbf{Obs. 8: ReCal reduces the residual damage identified before recovery.}\label{sec:damage_calibration}
On the fixed Q5 tokens, \method{} lowers KL at every stage, ending at 0.470 versus 0.627 for Base, a 25.1\% reduction. The paired difference is $-0.157$ nats (95\% interval: $[-0.177,-0.140]$). Q3--Q5 improve after OPD, while Q1/Q2 intervals include zero. The reduction thus reaches the vulnerable tokens identified in Section~\ref{sec:analysis}, without improving every group.

\item \textbf{Obs. 9: The average SFT advantage persists after OPD.}\label{sec:recovery}
Across all eight pairs, the mean AIME advantage is 7.9 points after SFT and 7.7 after OPD. The signed advantage increases in four pairs and decreases in four. In 8B LLM-Pruner, OPD narrows the SFT deficit without reversing it. Final gains therefore do not imply uniformly larger OPD increments. Appendix~\ref{app:extra_diagnostics} adds completion and alignment diagnostics.
\end{itemize}

\section{Related Work}\label{sec:related}
\paragraph{Structured pruning and capability recovery.}
Structured pruning selects channels or heads using gradient, activation, or reconstruction criteria \citep{ma2023llmpruner,an2023flap,ouden2024surgeon,meng2024osscar,ling2024slimgpt,guo2025slimllm}. Continued training and distillation recover the resulting models \citep{xia2023sheared,muralidharan2024compact,sreenivas2024minitron}. DarwinLM evaluates candidate structures after brief training, while NIRVANA models fine-tuning dynamics and selects calibration data with KL \citep{tang2026darwinlm,ai2025nirvana}. Reasoning-oriented approaches use generated traces, preserve sampling diversity, or improve student initialization \citep{wang2025think,nguyen2025sprint,monjur2026revisiting,he2026red}. \method{} measures teacher--probe damage at individual tokens and reweights their contribution within each teacher trajectory. It preserves the base criterion and widths, changing the calibration evidence used to select a student for SFT and OPD. Appendix~\ref{app:extended_related} compares these mechanisms and broader pruning approaches.

\paragraph{Offline and on-policy distillation.}
Offline distillation fits teacher responses \citep{kim2016sequence}; OPD supervises student-generated prefixes to address training--generation mismatch \citep{agarwal2023opd,gu2023minillm}. DistiLLM uses skew divergences and adaptive reuse, DistiLLM-2 distinguishes teacher and student responses, and speculative distillation interleaves their tokens \citep{ko2024distillm,ko2025distillm2,xu2025speculative}. Recent work studies thinking-pattern compatibility, entropy-aware forward KL, and privileged-context self-distillation \citep{li2026rethinking,jin2026entropy,zhao2026selfdistilled}. For pruned students, ShortOPD schedules short-to-long rollouts to reduce computation spent on teacher-confirmed repetition \citep{zhang2026shortopd}. These methods modify recovery supervision or sampling. \method{} uses forward KL on teacher prefixes to select the starting structure, leaving the OPD loss unchanged.

\Needspace{24\baselineskip}
\section{Limitations}\label{sec:limitations}

\paragraph{Generalization beyond the current pruning setting.}
Our experiments focus on FFN-width pruning for two Qwen3 models followed by a fixed SFT--OPD recovery pipeline. While \method{} consistently improves recovered reasoning performance across several pruning criteria, its effectiveness may depend on the type of structure being removed and the capability being evaluated. In particular, the mixed results on instruction following and the 8B LLM-Pruner case indicate that recovery-aware calibration does not uniformly benefit every criterion or task. An important next step is to extend \method{} beyond FFN channels to other structured compression settings, such as attention-head and layer pruning, and to study whether the same recovery-aware signal transfers across model families and more interactive agent workloads.

\paragraph{Understanding what makes recovery-aware calibration effective.}
\method{} combines teacher trajectories with damage-aware weighting based on forward KL, so the current experiments evaluate the complete calibration procedure rather than isolating each design choice independently. Although the fixed-prefix analysis shows that \method{} reduces residual discrepancy on tokens that are heavily affected by pruning, the precise connection between this reduction and downstream recovery remains only partially understood. We are therefore interested in disentangling the roles of trajectory choice, divergence direction, and token weighting, and in developing diagnostics that follow the student states encountered during OPD rather than only fixed teacher prefixes. These directions may help identify when recovery-aware calibration is most useful and lead to more general principles for selecting structures that remain recoverable after pruning.

\section{Conclusion}\label{sec:conclusion}

We introduced \method{}, a recovery-aware calibration approach for structured pruning followed by on-policy distillation. By using teacher--probe disagreement to reweight calibration evidence before structure selection, \method{} improves recovered mathematical reasoning and reduces residual damage at vulnerable tokens. More broadly, our results suggest that pruning should account not only for what a model preserves immediately after compression, but also for what the resulting student can recover afterward.

\section*{AI Use Statement}\label{sec:ai_use}
Generative AI tools were used solely as assistants for coding and writing. For coding, they supported the implementation, debugging, refactoring, and documentation of author-designed training and evaluation pipelines, data-processing utilities, and visualization scripts. For writing, they were used to improve clarity, concision, phrasing, and copy-editing of content and claims developed by the authors. Generative AI was not used for research ideation, problem formulation, methodological or experimental design, theoretical development or proofs, literature search or reference selection, result interpretation, or the generation of experimental results. All AI-assisted code was reviewed and tested by the authors, and all reported results were independently verified against the underlying per-sample records. The authors reviewed all AI-assisted outputs and take full responsibility for the final manuscript, claims, code, and associated artifacts.

\section*{Ethics Statement}
This study examines compression and recovery of existing language models using public benchmarks and teacher-generated responses. Compression may reduce the resources needed to deploy capable models, including in applications with potential for misuse. The reported recovery gains do not establish safety, fairness, or reliability in deployment. In particular, instruction-following and scientific reasoning do not improve uniformly, and the compressed students retain failure modes of both pruning and their teachers. Models and data should be used under their applicable licenses, and intended applications require separate capability and safety evaluation.

\section*{Reproducibility Statement}
Our code: \url{https://github.com/jianghoucheng/ReCal}. Sections~\ref{sec:method} and~\ref{sec:setup} describe the method and evaluation comparisons. Appendix~\ref{app:protocol} documents the models, calibration adapters, training data, recovery settings, and evaluation protocol; Appendix~\ref{app:diagnostics} specifies the fixed-prefix analysis. The accompanying source bundle includes aggregate measurements and table/figure scripts. Appendix~\ref{app:theory} states the theoretical assumptions and provides the proof.

\bibliography{iclr2027_conference}
\bibliographystyle{iclr2027_conference}

\clearpage
\appendix
\makeatletter
\setlength{\@fptop}{0pt}
\setlength{\@fpsep}{12pt}
\setlength{\@fpbot}{0pt plus 1fil}
\makeatother
\raggedbottom
\section{Experimental Details}\label{app:protocol}
We describe the models and channel-scoring adapters in Sections~\ref{app:models}--\ref{app:adapters}, followed by the training data (Section~\ref{app:data}) and recovery settings (Section~\ref{app:training}). Sections~\ref{app:evaluation} and~\ref{app:diagnostics} specify benchmark evaluation and fixed-prefix analysis.

\subsection{Models and Pruned Architecture}\label{app:models}
We use Qwen3-8B with thinking enabled and Qwen3-4B-Instruct-2507 in non-thinking mode \citep{yang2025qwen3}. Each student's corresponding unpruned checkpoint serves as its frozen teacher. Structured pruning removes channels from the gated FFN. For hidden input $z$ and intermediate activation
\begin{equation}
a=\operatorname{SiLU}(W^{\mathrm{gate}}z)\odot W^{\mathrm{up}}z,\qquad o=W^{\mathrm{down}}a,
\end{equation}
removing channel $j$ removes row $j$ of $W^{\mathrm{gate}}$ and $W^{\mathrm{up}}$ and column $j$ of $W^{\mathrm{down}}$. Attention, embeddings, normalization layers, and the language-model head are retained. Every layer uses the same retained width, aligned to a multiple of 128. Table~\ref{tab:widths} gives the exact widths; whole-model parameter reduction is smaller than FFN-width reduction because only FFNs are pruned.

\begin{table}[t]
\centering\small
\caption{Retained FFN widths and whole-model parameter reduction. The original widths are 12,288 for 8B and 9,728 for 4B.}
\label{tab:widths}
\begin{tabular}{lrrrr}
\toprule
Model & Target reduction & Retained width & Actual reduction & Parameters\\
\midrule
8B & 15\% & 9,472 & 15.21\% & 6.945B\\
8B & 25\% & 7,680 & 24.89\% & 6.152B\\
8B & 35\% & 5,760 & 35.26\% & 5.303B\\
4B & 15\% & 7,552 & 14.96\% & 3.421B\\
4B & 25\% & 6,144 & 24.63\% & 3.032B\\
4B & 35\% & 4,608 & 35.19\% & 2.607B\\
\bottomrule
\end{tabular}
\end{table}

\subsection{Pruning Criteria and Calibration}\label{app:adapters}
The compared methods share the architecture in Section~\ref{app:models}. They differ in channel scores, rather than layerwise width allocation. For response $r$, let $a_{r,t,j}$ be channel $j$'s activation at token $t$ and $W^{\mathrm{down}}_{:,j}$ its output-weight column. We suppress the layer index below. For \method{}, define the weighted activation expectation
\begin{equation}\label{eq:weighted_expectation}
\mathbb E_w[f(a_j)]=\frac1R\sum_{r=1}^R\sum_{t\in\mathcal R_r}\widetilde w_{r,t}f(a_{r,t,j}),
\end{equation}
where $\mathcal R_r$ contains valid response tokens and each response has total weight one. The ordinary adapters aggregate uniformly over their prompt calibration inputs. \method{} instead uses teacher responses and the damage weights in Equation~\ref{eq:weights}.

\paragraph{Minitron.}
Following activation-based width pruning \citep{muralidharan2024compact,sreenivas2024minitron}, the implemented adapter first averages absolute activations within each calibration sequence, then aggregates their squares across sequences:
\begin{equation}
s_j=\sqrt{\sum_r\left(\frac1{|\mathcal T_r|}\sum_{t\in\mathcal T_r}|a_{r,t,j}|\right)^2}.
\end{equation}
Here $\mathcal T_r$ is the baseline sequence's token set. \method{} replaces the inner mean with $\sum_{t\in\mathcal R_r}\widetilde w_{r,t}|a_{r,t,j}|$. Squaring remains outside the within-sequence aggregation; this is not equivalent to a global mean-absolute-activation score.

\paragraph{Wanda.}
Wanda combines weight magnitude and activation statistics \citep{sun2023wanda}. Our structured channel adapter, Wanda-SP, scores
\begin{equation}
s_j=\|W^{\mathrm{down}}_{:,j}\|_2\sqrt{\mathbb E[a_j^2]}.
\end{equation}
\method{} substitutes $\mathbb E_w[a_j^2]$ for the uniform second moment. This preserves the weight factor and channel grouping. It is a structured adaptation of Wanda, whose original formulation prunes individual weights.

\paragraph{FLAP.}
The reported FLAP adapter uses the WIFN score \citep{an2023flap}:
\begin{equation}
s_j=\|W^{\mathrm{down}}_{:,j}\|_2\mathbb E[|a_j|].
\end{equation}
\method{} replaces the activation expectation with $\mathbb E_w[|a_j|]$. Although the implementation also records the WIFV statistic $\|W^{\mathrm{down}}_{:,j}\|_2^2\operatorname{Var}(a_j)$, WIFV does not select the reported masks. We do not use FLAP's adaptive layerwise sparsity or bias compensation; fixed per-layer widths and the same recovery pipeline allow comparison of channel selection across all adapters.

\paragraph{LLM-Pruner.}
The first-order LLM-Pruner adapter groups the gate, up, and down weights associated with each channel, denoted by $G_j$, and uses the first-order score \citep{ma2023llmpruner}
\begin{equation}
s_j=\sum_{u\in G_j}\left|u\frac{\partial\mathcal L}{\partial u}\right|.
\end{equation}
The baseline uses next-token cross entropy on calibration prompts. \method{} differentiates a weighted response-token loss, $\mathcal L=\sum_t w_t\operatorname{CE}_t/\sum_t w_t$, and applies the same group score. Only FFN projection parameters receive gradients during scoring. The maximum scoring sequence length is 2,048 for both variants. Weights enter the loss before differentiation and before the absolute-value aggregation.

\paragraph{Calibration data and divergence.}
Both calibration pools are balanced equally across mathematics, code, science, and instruction following. The baseline pool contains prompts, whereas \method{} uses responses from the unpruned teacher, with prompt and padding tokens excluded from weighted statistics. The probe is the corresponding baseline pruned to the same target widths, without recovery training. Final scores are recomputed on the original unpruned model; probe weights and masks are not inherited by the final student.

For efficient damage estimation, let $S_{r,t}$ be the teacher's Top-$k$ token set at a shared prefix. We retain the original probabilities on these tokens and collect the remaining mass in a single bucket:
\begin{equation}\label{eq:coarse_kl}
\widehat d_{r,t}=\sum_{v\in S_{r,t}}p_T(v)\log\frac{p_T(v)}{p_P(v)}
+p_T(\mathrm{other})\log\frac{p_T(\mathrm{other})}{p_P(\mathrm{other})},
\end{equation}
where $p_M(\mathrm{other})=1-\sum_{v\in S_{r,t}}p_M(v)$ for $M\in\{T,P\}$. We use $k=16$ and a maximum calibration length of 8,192 tokens for the activation-based adapters. This is KL on a shared coarsened vocabulary, not KL after renormalizing only the Top-$k$ tokens. The per-response weights are $\widetilde w_{r,t}=\widehat d_{r,t}/\sum_u\widehat d_{r,u}$. If the denominator is at most $10^{-9}$, we use uniform weights over valid response tokens. Each response therefore has equal total weight regardless of its length. Damage scores are aligned with the response-token activations that predict the affected next token before statistics are accumulated. These comparisons evaluate the combined use of teacher trajectories and damage weighting.

\subsection{Training Data}\label{app:data}
SFT and OPD use separate prompt pools drawn from mathematics, code, science, and instruction following, initially mixed in equal proportions (1:1:1:1). Table~\ref{tab:training_sources} lists the sources. Science prompts balance physics, chemistry, and biology in a 1:1:1 ratio; OPD instruction prompts balance one through five constraints. Mathematics prompts added from the SFT source use a separate source partition, and the final SFT and OPD pools have no normalized-prompt overlap. Prompts undergo deduplication and benchmark-overlap screening.

\begin{table}[t]
\centering\small
\setlength{\tabcolsep}{4pt}\renewcommand{\arraystretch}{1.13}
\caption{Public prompt sources. Dataset identifiers are linked to their source pages. NVIDIA datasets carry the \texttt{nvidia/} namespace; IF-multi is from \texttt{allenai/}. The initial domain mixture is 1:1:1:1 in both stages.}
\label{tab:training_sources}
\begin{tabularx}{\linewidth}{@{}l>{\raggedright\arraybackslash}p{.36\linewidth}>{\raggedright\arraybackslash}X@{}}
\toprule
Domain & SFT prompt source & OPD prompt source\\
\midrule
Math & \href{https://huggingface.co/datasets/nvidia/Nemotron-SFT-Math-v4}{Nemotron-SFT-Math-v4} & \href{https://huggingface.co/datasets/nvidia/Nemotron-RL-Math-v2}{Nemotron-RL-Math-v2}; disjoint supplement from SFT-Math-v4\\
Code & \href{https://huggingface.co/datasets/nvidia/Nemotron-SFT-Competitive-Programming-v2}{Nemotron-SFT-Competitive-Programming-v2} & \href{https://huggingface.co/datasets/nvidia/Nemotron-RL-coding-competitive_coding}{Nemotron-RL-coding-competitive\_coding}\\
Science & \href{https://huggingface.co/datasets/nvidia/Nemotron-SFT-Science-v2}{Nemotron-SFT-Science-v2} & \href{https://huggingface.co/datasets/nvidia/Nemotron-RL-Science-v1}{Nemotron-RL-Science-v1}\\
Instruction & \href{https://huggingface.co/datasets/nvidia/Nemotron-SFT-Instruction-Following-Chat-v3}{Nemotron-SFT-Instruction-Following-Chat-v3} & \href{https://huggingface.co/datasets/allenai/IF_multi_constraints_upto5}{IF\_multi\_constraints\_upto5}\\
\bottomrule
\end{tabularx}
\end{table}

Only prompts are retained from these sources. For SFT, the corresponding unpruned teacher generates new targets with temperature 0.6 and top-$p$ 0.95. Truncated or malformed responses are excluded, so the final SFT response mixture is not exactly domain-balanced. All compared methods for a given model share this same target set. Mathematics and science use a common final-answer format, code follows the LiveCodeBench Python template, and instruction-following prompts retain their original constraints. OPD uses prompts only; it does not consume source answers or external task rewards.

\subsection{Recovery Hyperparameters}\label{app:training}
Table~\ref{tab:hyperparameters} summarizes the settings that determine training and sampling. Both stages use AdamW \citep{loshchilov2019adamw}. SFT minimizes next-token cross entropy on assistant responses, including reasoning tokens when thinking is enabled. User and system tokens do not contribute to the loss. OPD then samples fresh responses from the current student and evaluates the teacher at those student-generated prefixes. All completed main runs use 200 OPD updates.

\begin{table}[t]
\centering\small\renewcommand{\arraystretch}{1.13}
\caption{Recovery settings shared by each Base/ReCal pair. SFT is offline distillation on teacher responses. OPD uses student rollouts.}
\label{tab:hyperparameters}
\begin{tabularx}{\linewidth}{@{}lYY@{}}
\toprule
Setting & SFT & OPD\\
\midrule
Optimizer & AdamW & AdamW\\
Learning rate & $10^{-5}$ & $10^{-6}$\\
Schedule & Cosine, 5\% warmup & Constant\\
Batch size & 32 sequences & 64 prompts\\
Training budget & 2 epochs & 200 updates\\
Responses per prompt & 1 teacher target & 4 student rollouts\\
Sampling temperature & 0.6 (teacher) & 1.0 (student)\\
Sampling top-$p$ & 0.95 & 1.0\\
Response limit & 20,480 & 20,480\\
Sequence / prompt limit & 21,504 sequence & 2,048 prompt\\
Distillation support & Sampled teacher tokens & Student Top-16\\
Teacher temperature & --- & 1.0\\
Repetition penalty & --- & 1.05\\
Optimization passes per update & --- & 1\\
Policy clip ratio & --- & 0.2\\
Gradient clipping & 1.0 & 1.0\\
Weight decay & 0 & 0\\
\bottomrule
\end{tabularx}
\end{table}

\paragraph{Restricted-support OPD signal.}
Equation~\ref{eq:opd} describes the full-vocabulary reference objective. The implemented signal instead uses the student's Top-16 token IDs $K_t$ at each prefix $h_t$. Student probabilities are normalized on this support:
\begin{equation}
\widetilde p_S(v\mid h_t)=\frac{p_S(v\mid h_t)}{\sum_{u\in K_t}p_S(u\mid h_t)},\qquad v\in K_t.
\end{equation}
The token-level teacher signal is
\begin{equation}\label{eq:opd_signal}
r_{t,v}=-\widetilde p_S(v\mid h_t)\bigl[\log p_S(v\mid h_t)-\log p_T(v\mid h_t)\bigr].
\end{equation}
The implementation uses this signal directly as an advantage in its clipped policy update, without GAE or an external task reward. The log probabilities retain their original vocabulary normalization. Thus, this is a restricted-support reverse-KL-style signal, not the full KL in Equation~\ref{eq:opd}, and it differs from the teacher-support-plus-residual calculation used for calibration.

\paragraph{Shared optimization conditions.}
Paired runs start from the same unpruned checkpoint, retain identical per-layer widths, and share the same SFT targets, OPD prompts, data ordering, and optimization settings. Training uses FSDP \citep{zhao2023fsdp}, BF16 forward computation, and FP32 master weights. The recorded environment uses eight NVIDIA H20 GPUs. Matching update counts does not imply identical sampled token counts or GPU time, since students can generate different response lengths.

\subsection{Benchmarks and Evaluation}\label{app:evaluation}
We evaluate every available checkpoint before recovery, after SFT, and after SFT plus OPD, using the same response budget of 20,480 tokens. AIME is sampled at temperature 0.6 and top-$p$ 0.95; the other benchmarks use greedy decoding. Metrics are defined below.

\paragraph{AIME 2024--2026.}
Each annual set contains competition mathematics problems with exact numerical answers. We average correctness over eight sampled responses per question and then average the three yearly scores. This is average sampled accuracy, not pass@8. The final boxed answer is extracted and normalized before comparison.

\paragraph{GPQA-Diamond.}
GPQA evaluates graduate-level science questions in biology, chemistry, and physics \citep{rein2023gpqa}. We use the Diamond test split and report multiple-choice accuracy, extracting the final A/B/C/D answer from the common boxed-answer format.

\paragraph{IFBench.}
IFBench evaluates compliance with verifiable instructions \citep{pyatkin2025ifbench}. We use the official test set and verifier and report instruction-level strict accuracy (IFB-S). A response can satisfy some constraints while violating others, so this metric is distinct from prompt-level strict accuracy.

\paragraph{LiveCodeBench.}
We use the code-generation task in LiveCodeBench-v5, covering problems released from October 2024 through February 2025 \citep{jain2024livecodebench}. We report pass@1 under the official input/output or function-call templates, with a 10-second execution timeout per problem.

\paragraph{Reporting conventions.}
Scores are percentages. Within each main-table model block, boldface marks the highest displayed score in each column and underlining the second-highest distinct score, across all pruned stages, excluding the unpruned reference. Ties receive the same formatting. The final Avg. column is the unweighted arithmetic mean of AIME24, AIME25, AIME26, GPQA, IFBench, and LiveCodeBench-v5. AIME averages elsewhere use only the three AIME benchmarks. The original seven paired gains were reported before rounding and can differ by 0.1 percentage point from displayed-score differences. The additional 8B LLM-Pruner gain is computed from the reported three-decimal scores. Recovery-stage averages use full-precision repository records for all eight pairs; final-performance comparisons also use all eight. The unpruned 8B AIME/GPQA reference was regraded at the common response limit; IFBench and LiveCodeBench use fresh generations. Individual-year AIME scores are available for every main-table checkpoint and both LLM-Pruner OPD pairs. The 8B LLM-Pruner + \method{} Pruned/SFT scores are included in Table~\ref{tab:additional_early}.

\subsection{Fixed-Prefix Diagnostic Protocol}\label{app:diagnostics}
We use held-out teacher responses balanced across the four domains to measure distributional damage. Both teacher and student evaluate the same prefix at every response token. Within each response, tokens are ranked by raw-pruned baseline KL and divided into five equal-count groups, Q1--Q5, from lowest to highest damage. The groups are fixed thereafter; no checkpoint selects its own favorable token subset.

Concentration curves average each response's fraction of damage carried by its highest-damage tokens. The fixed-group plots first average KL within each response/group, then across responses. For the Q5 damage share, each response's group-average KL is multiplied by its token count, divided by total response damage, and then averaged across responses. This is an average of per-response ratios, not a ratio of group means.

The diagnostic pool contains 128 responses per model, equally split across domains, with no filtering by answer correctness or damage. Recovery tracking uses 8B Wanda at the main compression ratio. Section~\ref{sec:analysis} uses only baseline checkpoints; Section~\ref{sec:ablation} adds the calibrated counterparts on the same responses and token groups. Intervals use 2,000 trajectory-bootstrap resamples; differences use paired resampling. Tokens within a response are not treated as independent samples.

\section{Supplementary Experiments}\label{app:supplementary}
We supplement the main results with the LLM-Pruner integration in Section~\ref{app:llmpruner}, additional damage analysis in Section~\ref{app:extra_damage}, and generation diagnostics in Section~\ref{app:extra_diagnostics}. Section~\ref{app:intro_design} documents the introduction diagnostics, and Section~\ref{app:additional_results} reports benchmark results not displayed in the main paper.

\subsection{LLM-Pruner Integration}\label{app:llmpruner}
Figure~\ref{fig:llmpruner} extends the comparison to both LLM-Pruner models. In 4B, post-OPD AIME, GPQA, and code improve, while IFBench declines slightly. In 8B, AIME falls from 12.4\% to 9.2\% and IFBench from 22.4\% to 19.5\%, whereas GPQA rises from 34.8\% to 44.9\% and code from 17.5\% to 28.3\%. The 8B AIME24/25/26 scores change from 18.8/7.9/10.4\% to 13.8/5.0/8.7\%, so the mathematics decline occurs in every year. Thus, the positive 4B result does not generalize uniformly to 8B. The 8B mathematics deficit is already present after SFT: 3.2\% for the calibrated student versus 7.9\% for Base. OPD narrows this deficit from 4.7 to 3.2 points but does not reverse it. As with the other adapters, the comparison changes both calibration context and weighting.
\begin{figure}[t]
\centering\includegraphics[width=\linewidth]{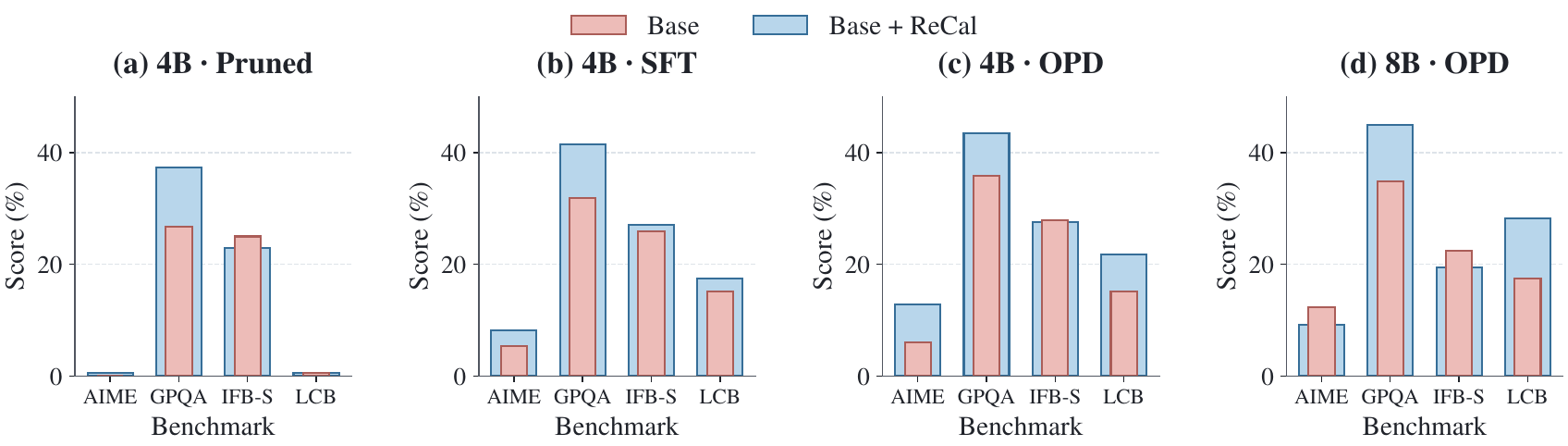}
\caption{\textbf{LLM-Pruner with and without \method{}.} (a--c) The three 4B recovery stages. (d) The completed 8B OPD pair. AIME denotes the three-year average. Wide blue bars show Base + \method{}; narrow red bars show Base.}
\label{fig:llmpruner}
\end{figure}

\subsection{Additional Damage Analysis}\label{app:extra_damage}
Figure~\ref{fig:analysis_supp} adds the 4B concentration curves and paired KL changes before and after SFT. The concentration pattern extends to 4B. In 8B Wanda, the paired changes reveal increased discrepancy in some lower-damage groups after pruning and SFT, despite the reduction in Q5. These earlier stages complement the post-OPD comparison in Figure~\ref{fig:recovery_evidence}(b).
\begin{figure}[t]
\centering\includegraphics[width=\linewidth]{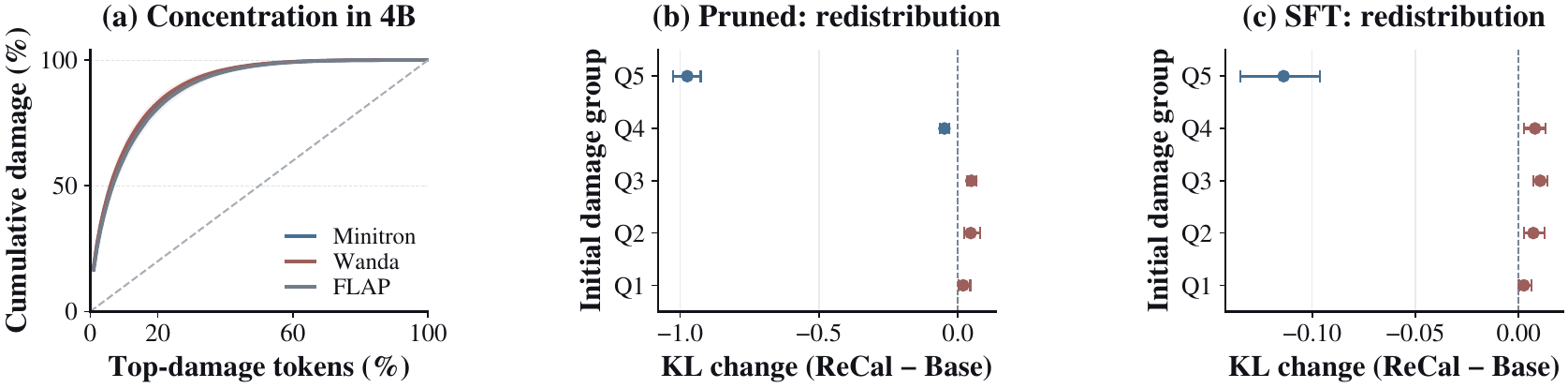}
\caption{\textbf{Additional concentration and redistribution results.} (a) Within-response damage concentration in 4B. (b,c) Paired \method{} minus Base KL after pruning and SFT. Groups remain fixed by the raw-pruned baseline. Intervals are 95\% paired trajectory-bootstrap intervals.}
\label{fig:analysis_supp}
\end{figure}

\subsection{Generation and Alignment Diagnostics}\label{app:extra_diagnostics}
Figure~\ref{fig:supp_recovery}(a) reports net completion rescue on paired AIME generations: the number where Base truncates but \method{} completes, minus the number with the reverse outcome. Pairing uses the same question and sample index, not an identical sampled trajectory. All seven post-OPD pairs with completion diagnostics are positive, ranging from 12 to 116 net completions per 720 generations. However, 4B Wanda and FLAP initially have fewer completed generations immediately after pruning. Completion after pruning is therefore an incomplete indicator of the eventual recovered model. Its aggregate post-OPD truncation rate nevertheless increases from 76.8\% to 79.4\% with \method{}, consistent with its mathematics exception.

\begin{figure}[t]
\centering\includegraphics[width=\linewidth]{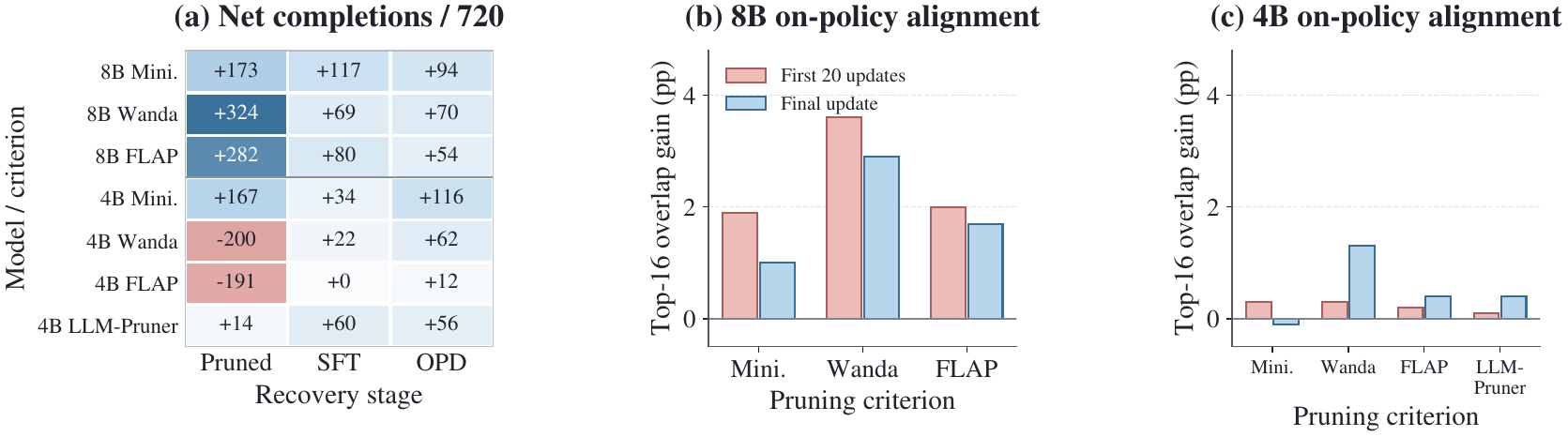}
\caption{\textbf{Supplementary recovery diagnostics.} (a) Net completion rescue across the three stages. (b,c) Top-16 teacher--student overlap gains during the first 20 OPD updates and at the final update. Each model generates its own trajectories.}
\label{fig:supp_recovery}
\end{figure}

The three 8B comparisons enter OPD with 1.9--3.6 percentage points higher teacher--student Top-16 overlap and finish 1.0--2.9 points higher. The 4B changes are smaller and not uniformly positive. Because these statistics use each student's own rollouts, they describe the trajectories encountered during recovery and do not isolate agreement at matched states.

Trajectory-level diagnostics further distinguish response length from token damage. Across probes, response length correlates strongly with summed KL (Spearman 0.783--0.886), motivating equal total weight per response. Cross-probe mean-damage rankings are highly correlated (0.963--0.975), so the necessity of a criterion-specific probe is not established. Instruction responses have high mean damage, yet IFBench does not consistently improve. A high divergence is therefore not interchangeable with task importance.

\subsection{Introduction Figure Measurement}\label{app:intro_design}
Figure~\ref{fig:intro}(b,c) uses the measured Qwen3-8B Minitron, Wanda, and FLAP baselines at the main compression ratio. Panel (b) reports mean AIME accuracy across 2024--2026 for the unpruned teacher, raw-pruned student, and post-SFT student. The shared teacher score is 61.4\%. Panel (c) reports the fraction of AIME generations reaching the common response-length limit after SFT and after OPD. Scores follow the sampling and decoding settings in Section~\ref{app:evaluation}.

Truncation measures failure to finish within the evaluation budget; it does not distinguish repetition, an unfinished derivation, or a completed answer followed by excess text. It is therefore reported separately from answer accuracy. These aggregate measurements do not estimate reverse KL conditional on failure type. The introduction's discussion of teacher confirmation at flawed prefixes draws on prior work \citep{zhang2026shortopd}; our fixed-prefix experiments in Section~\ref{sec:analysis} examine pruning damage directly.

\subsection{Additional Benchmark Results}\label{app:additional_results}
Tables~\ref{tab:additional_early} and~\ref{tab:additional_opd} report the calibrated Minitron, FLAP, and additional LLM-Pruner scores not displayed in Table~\ref{tab:main} or Figure~\ref{fig:compatibility}. The first table adds Pruned/SFT checkpoints; the second adds post-OPD science and instruction scores. Matching baseline scores are in Table~\ref{tab:main}. The 4B/8B LLM-Pruner comparisons are shown in Figure~\ref{fig:llmpruner}. Complete numerical records, including the ratio sweep, are retained in the accompanying data files.
\begin{table}[t]
\centering\small\setlength{\tabcolsep}{3pt}\renewcommand{\arraystretch}{1.10}
\caption{Additional Pruned and SFT scores at 25\% parameter reduction. These calibrated variants are absent from the main table; their post-OPD scores appear in Figures~\ref{fig:compatibility} and~\ref{fig:llmpruner}. Scores are percentages.}
\label{tab:additional_early}
\begin{tabularx}{\linewidth}{@{}ll*{6}{Y}@{}}
\toprule
Method & Stage & AIME24 & AIME25 & AIME26 & GPQA & IFB-S & LCB-v5\\
\midrule
\multicolumn{8}{c}{\textbf{Qwen3-8B}}\\
\midrule
Minitron + \method{} & Pruned & 4.6 & 0.8 & 1.7 & 30.8 & 18.6 & 5.4\\
 & SFT & 25.8 & 16.7 & 18.3 & 35.4 & 18.6 & 17.5\\
FLAP + \method{} & Pruned & 15.0 & 9.6 & 7.5 & 28.8 & 18.6 & 8.4\\
 & SFT & 41.2 & 30.8 & 32.1 & 38.4 & 20.1 & 21.1\\
LLM-Pruner + \method{} & Pruned & 0.0 & 0.0 & 0.0 & 27.3 & 18.0 & 0.0\\
 & SFT & 6.3 & 0.4 & 2.9 & 41.9 & 17.2 & 19.3\\
\midrule
\multicolumn{8}{c}{\textbf{Qwen3-4B-Instruct}}\\
\midrule
Minitron + \method{} & Pruned & 4.2 & 1.7 & 3.7 & 30.3 & 23.3 & 6.6\\
 & SFT & 22.9 & 15.0 & 15.8 & 37.9 & 25.6 & 16.3\\
FLAP + \method{} & Pruned & 14.6 & 10.0 & 8.3 & 32.8 & 25.3 & 11.4\\
 & SFT & 27.9 & 23.7 & 18.3 & 34.8 & 24.4 & 18.7\\
\bottomrule
\end{tabularx}
\end{table}
\begin{table}[t]
\centering\small\setlength{\tabcolsep}{8pt}\renewcommand{\arraystretch}{1.12}
\caption{Additional post-OPD science and instruction results at 25\% reduction, complementing the AIME and code results in Figure~\ref{fig:compatibility}. Scores are percentages.}
\label{tab:additional_opd}
\begin{tabular}{@{}llrr@{}}
\toprule
Model & Method & GPQA & IFB-S\\
\midrule
8B & Minitron + \method{} & 36.4 & 19.5\\
8B & FLAP + \method{} & 37.4 & 19.8\\
4B & Minitron + \method{} & 36.4 & 27.0\\
4B & FLAP + \method{} & 35.4 & 27.0\\
\bottomrule
\end{tabular}
\end{table}

\FloatBarrier

\section{Forward KL and Teacher-Trajectory Preservation}\label{app:theory}
We first relate full-vocabulary forward KL to preservation of teacher sequence probabilities in Section~\ref{app:retention}. Section~\ref{app:bound_scope} explains why this connection motivates the method without guaranteeing recovery gains.

\subsection{Sequence-Distribution Connection}\label{app:retention}
For a fixed prompt $x$, write $P_T$ and $P_m$ for teacher and pruned-model sequence distributions up to a shared horizon $H$, with an absorbing EOS state. Define
\begin{equation}\label{eq:trajectory_damage}
D(m)=\mathbb E_{y\sim P_T}\sum_{t=1}^H\KL\bigl(p_T(\cdot\mid x,y_{<t})\Vert p_m(\cdot\mid x,y_{<t})\bigr).
\end{equation}
\begin{proposition}[Teacher-trajectory preservation]\label{prop:retention}
Assume $P_m$ is positive wherever $P_T$ is positive. For any set $A$ of completed responses,
\begin{equation}\label{eq:retention}
\KL(P_T\Vert P_m)=D(m),\qquad P_m(A)\geq P_T(A)-\sqrt{D(m)/2}.
\end{equation}
\end{proposition}
\begin{proof}
Both distributions use the same vocabulary and stopping convention. The autoregressive factorization and the KL chain rule \citep{cover2006information} give
\begin{align}
\KL(P_T\Vert P_m)
&=\mathbb E_{y\sim P_T}\log\frac{\prod_{t=1}^H p_T(y_t\mid x,y_{<t})}{\prod_{t=1}^H p_m(y_t\mid x,y_{<t})}\\
&=\sum_{t=1}^H\mathbb E_{y_{<t}\sim P_T}\KL\bigl(p_T(\cdot\mid x,y_{<t})\Vert p_m(\cdot\mid x,y_{<t})\bigr)=D(m).
\end{align}
Terms after EOS are zero. With natural logarithms, Pinsker's inequality yields
\begin{equation}
|P_T(A)-P_m(A)|\leq\TV(P_T,P_m)\leq\sqrt{\KL(P_T\Vert P_m)/2},
\end{equation}
which proves the claim. The same argument applies to a shared distribution over prompts by considering the joint prompt--response distribution. This proposition applies standard information-theoretic identities to the pruning setting; it is not a new KL inequality.
\end{proof}

\subsection{Interpretation and Scope}\label{app:bound_scope}
The connection explains why teacher-generated prefixes and teacher-weighted disagreement are relevant to retaining teacher-supported continuations. It does not show that KL weighting directly minimizes $D(m)$, that smaller bounds order actual task accuracies, or that SFT and OPD preserve any ordering of masks. The teacher may also generate incorrect answers.

The population quantity $D(m)$ sums token divergences before averaging over responses. The fixed-group means and length-normalized response averages in our diagnostics are different statistics. Moreover, merging tail tokens into the residual bucket in Equation~\ref{eq:coarse_kl} can only decrease KL by the data-processing inequality. The implemented coarsened KL is thus not an upper bound on full-vocabulary damage and cannot be substituted into Proposition~\ref{prop:retention} to certify sequence preservation. Recovery benefits are established empirically through the paired experiments.

\section{Additional Related-Work Comparisons}\label{app:extended_related}
We extend Section~\ref{sec:related} by distinguishing the structural decisions, calibration signals, and recovery interventions in prior work. These distinctions determine which methods are direct comparisons and which address complementary problems.

\paragraph{Pruning granularity and selection criteria.}
SparseGPT removes individual weights using an approximate second-order reconstruction procedure, while Wanda combines weight magnitude with input activations \citep{frantar2023sparsegpt,sun2023wanda}. Their original weight-level formulations differ from removing complete FFN channels; our Wanda-SP adaptation is specified in Appendix~\ref{app:adapters}. Among structured methods, the LLM Surgeon uses curvature information and weight updates, OSSCAR formulates layer reconstruction as combinatorial optimization, and SlimGPT applies batched greedy pruning within an Optimal Brain Surgeon framework \citep{ouden2024surgeon,meng2024osscar,ling2024slimgpt}. Bonsai estimates importance through forward-pass perturbations, avoiding backward passes \citep{kolawole2024bonsai}. These methods primarily differ in how they estimate or optimize the cost of deleting a structure. \method{} changes the token evidence supplied to a criterion. Compatibility with the four evaluated adapters does not establish compatibility with every reconstruction or curvature-based method.

Layer and block methods offer a different compression axis. ShortGPT ranks layer redundancy, LaCo collapses layers, BlockPruner distinguishes attention and feed-forward blocks, and SLEB eliminates redundant transformer blocks \citep{men2024shortgpt,ma2024laco,yang2024blockpruner,song2024sleb}. SliceGPT rotates representations before removing rows and columns \citep{ashkboos2024slicegpt}. ZipLM incorporates inference cost into structural decisions \citep{kurtic2023ziplm}. Our experiments retain all layers and attention modules and compare FFN channels at identical widths. Consequently, they do not establish the best compression axis or measured deployment speedups.

\paragraph{Accounting for recovery during compression.}
Sheared LLaMA combines targeted structured pruning with continued pretraining; Minitron studies pruning and distillation across multiple architectural dimensions \citep{xia2023sheared,muralidharan2024compact,sreenivas2024minitron}. More directly, DarwinLM incorporates brief candidate training into evolutionary architecture search \citep{tang2026darwinlm}. NIRVANA uses function-space saliency to account for fine-tuning dynamics and KL-based calibration-data selection \citep{ai2025nirvana}. These works already establish that compression should consider subsequent training. Our contribution is the specific teacher--probe token signal and its integration into calibration for an unchanged SFT--OPD pipeline, rather than the general idea of recovery-aware pruning. Unlike candidate training or layer-budget search, \method{} fixes the target widths and uses an unrecovered probe to redistribute calibration weight within responses.

\paragraph{Preserving reasoning.}
Reasoning models can exploit long generated traces and test-time computation \citep{deepseek2025r1,yang2025qwen3}, making next-token fit alone an incomplete evaluation of compression. RESP uses self-generated reasoning traces and decode-only gradient scoring, SPRINT considers diverse best-of-$N$ reasoning, and recent evaluations examine pruning under test-time scaling \citep{wang2025think,nguyen2025sprint,monjur2026revisiting}. RED uses activation-aware initialization for reasoning-model distillation \citep{he2026red}. MuCRASP extends reasoning-aware pruning to multimodal models using reasoning and modality signals \citep{dutta2026mucrasp}. Our fixed-prefix analysis instead tracks baseline-defined damaged tokens through recovery. It supports an empirical calibration strategy, without showing that large token divergence always identifies task-critical reasoning steps. Our mathematics and code gains, instruction-following losses, and the 8B LLM-Pruner exception bound that interpretation.

\paragraph{Distillation objectives and rollout distributions.}
Classical distillation transfers teacher predictions, while sequence-level distillation fits generated teacher responses \citep{hinton2015distilling,kim2016sequence}. Learning on learner-induced states addresses a general distribution-shift problem in sequential prediction \citep{ross2011dagger}; language-model OPD obtains teacher supervision on these states \citep{agarwal2023opd,gu2023minillm}. Speculative distillation interleaves student proposals and teacher replacements to balance on-policy coverage with response quality \citep{xu2025speculative}. Self-Distilled Reasoner obtains supervision from a solution-conditioned version of the same model \citep{zhao2026selfdistilled}. These supervision designs differ from our frozen unpruned teacher. Entropy-aware OPD also employs forward KL, but uses teacher uncertainty to adapt the recovery objective \citep{jin2026entropy}. \method{} uses teacher--probe disagreement before recovery to weight structural calibration. Thus, the divergence direction alone is not the novelty; its role in pre-recovery structure selection is the relevant distinction.

\end{document}